\documentclass{article}

\usepackage{microtype}
\usepackage{graphicx}
\usepackage{subfigure}
\usepackage{booktabs} % for professional tables

\usepackage{enumitem}
\usepackage{hyperref}

\usepackage[accepted]{icml2018}
\usepackage{comment}

\usepackage{amsmath,amssymb,amsfonts,graphicx,amsthm,nicefrac,mathtools,bm}
\usepackage{tabularx}
\usepackage{float}
\usepackage{soul}
\usepackage{color}

\DeclareRobustCommand{\hlyellow}[1]{{\sethlcolor{yellow}\hl{#1}}}

\newcommand{\powerset}{\raisebox{.15\baselineskip}{\Large\ensuremath{\wp}}}  
\newcommand{\Explain}{\ensuremath{\mathcal{E}}}

\newcommand{\Prob}{\ensuremath{\mathbb{P}}}
\newcommand{\Exp}{\ensuremath{\mathbb{E}}}
\newcommand{\ExpModel}{\ensuremath{\Exp_m}}
\newcommand{\real}{\ensuremath{\mathbb{R}}}
\newcommand{\ProbModel}{\ensuremath{\Prob_m}}
\newcommand{\ProbX}{\ensuremath{\Prob_X}}
\newcommand{\defn}{\ensuremath{: \, = }}

\newcommand{\Qprob}{\ensuremath{\mathbb{Q}}}
\newcommand{\Qfamily}{\ensuremath{\mathcal{Q}}}
\newtheoremstyle{slplain}% name
  {.5\baselineskip\@plus.2\baselineskip\@minus.2\baselineskip}% Space above
  {.5\baselineskip\@plus.2\baselineskip\@minus.2\baselineskip}% Space below
  {\slshape}% Body font
  {}%Indent amount (empty = no indent, \parindent = para indent)
  {\bfseries}%  Thm head font
  {.}%       Punctuation after thm head
  { }%      Space after thm head: " " = normal interword space;
  {}%       Thm head spec

\theoremstyle{slplain}

\newtheorem{theorem}{Theorem}

\icmltitlerunning{Learning to Explain: An Information-Theoretic Perspective on
Model Interpretation}

\begin{document}

% Modify the spaces above and below equations.
\setlength{\abovedisplayskip}{4pt}
\setlength{\belowdisplayskip}{4pt}

\twocolumn[
\icmltitle{Learning to Explain: An Information-Theoretic Perspective \\ on
Model Interpretation}

% It is OKAY to include author information, even for blind
% submissions: the style file will automatically remove it for you
% unless you've provided the [accepted] option to the icml2018
% package.

% List of affiliations: The first argument should be a (short)
% identifier you will use later to specify author affiliations
% Academic affiliations should list Department, University, City, Region, Country
% Industry affiliations should list Company, City, Region, Country

% You can specify symbols, otherwise they are numbered in order.
% Ideally, you should not use this facility. Affiliations will be numbered
% in order of appearance and this is the preferred way.
\icmlsetsymbol{equal}{*}

\begin{icmlauthorlist}
\icmlauthor{Jianbo Chen}{ucb,af1}
\icmlauthor{Le Song}{git,af}
\icmlauthor{Martin J. Wainwright}{ucb}
\icmlauthor{Michael I. Jordan}{ucb} 
\end{icmlauthorlist}

\icmlaffiliation{ucb}{University of California, Berkeley}
\icmlaffiliation{git}{Georgia Institute of Technology}
\icmlaffiliation{af}{Ant Financial}
\icmlaffiliation{af1}{Work done partially during an internship at Ant Financial}

\icmlcorrespondingauthor{Jianbo Chen}{jianbochen@berkeley.edu} 

% You may provide any keywords that you
% find helpful for describing your paper; these are used to populate
% the "keywords" metadata in the PDF but will not be shown in the document
\icmlkeywords{Model Interpretation}

\vskip 0.3in
]

% this must go after the closing bracket ] following \twocolumn[ ...

% This command actually creates the footnote in the first column
% listing the affiliations and the copyright notice.
% The command takes one argument, which is text to display at the start of the footnote.
% The \icmlEqualContribution command is standard text for equal contribution.
% Remove it (just {}) if you do not need this facility.

\printAffiliationsAndNotice{}  % leave blank if no need to mention equal contribution
% \printAffiliationsAndNotice{\icmlEqualContribution} % otherwise use the standard text.

\begin{abstract}
We introduce \emph{instancewise feature selection} as a methodology for model interpretation.  Our method is based on learning a function to extract a subset of features that are most informative for each given example.  This feature selector is trained to maximize the mutual information between selected features and the response
variable, where the conditional distribution of the response variable given the input is the model to be explained. We develop an efficient variational approximation to the mutual information, and show the effectiveness of our method on a variety of synthetic and real data sets using both quantitative metrics and human evaluation.

% and show that the
% resulting method compares favorably to other model explanation methods
% on a variety of synthetic and real data sets using both quantitative
% metrics and human evaluation.
\end{abstract}

%------------------------------------------------------------------------
%\vspace{-2mm}
\section{Introduction}  
%\vspace{-1mm}
%------------------------------------------------------------------------

\setlength{\abovedisplayskip}{3pt}
\setlength{\abovedisplayshortskip}{1pt}
\setlength{\belowdisplayskip}{3pt}
\setlength{\belowdisplayshortskip}{1pt}
\setlength{\jot}{3pt}
\setlength{\textfloatsep}{3pt}  

% What is model interpretation and why it is important. 

% Why instancewise feature selection is for model interpretation. (It is only a sub-component of model interpretation.)

% previous approach.

% our approach.
% Why mutual information is a good way to quantize the instancewise feature selection.

% Why we choose to use such method for approximating mutual information. 

% our approach first or related work first.

Interpretability is an extremely important criterion when a machine
learning model is applied in areas such as medicine, financial
markets, and criminal justice (e.g., see the discussion paper by
Lipton~(\cite{lipton2016mythos}), as well as references therein). Many
complex models, such as random forests, kernel methods, and deep
neural networks, have been developed and employed to optimize
prediction accuracy, which can compromise their ease of
interpretation.

In this paper, we focus on \textit{instancewise feature selection} as
a specific approach for model interpretation. Given a machine learning
model, instancewise feature selection asks for the importance scores
of each feature on the prediction of a given instance, and the
relative importance of each feature are allowed to vary across
instances. Thus, the importance scores can act as an explanation for
the specific instance, indicating which features are the key for the
model to make its prediction on that instance. A related concept in
machine learning is feature selection, which selects a subset of
features that are useful to build a good predictor for a specified
response variable~\cite{guyon2003introduction}. While feature
selection produces a global importance of features with respect to the
entire labeled data set, instancewise feature selection measures
feature importance locally for each instance labeled by the model.

Existing work on interpreting models approach the problem from two
directions. The first line of work computes the gradient of the output
of the correct class with respect to the input vector for the given
model, and uses it as a saliency map for masking the
input~\cite{simonyan2013deep,Springenberg2014Striving}. The gradient
is computed using a Parzen window approximation of the original
classifier if the original one is not
available~\cite{baehrens2010explain}.
% Then the saliency map itself \cite{baehrens2010explain}, or its elementwise product with the input vector \cite{bach2015pixel,DBLP:journals/corr/KindermansSMD16}, is used as importance scores for each feature.d
Another line of research approximates the model to be interpreted via
a locally additive model in order to explain the difference between
the model output and some ``reference'' output in terms of the
difference between the input and some ``reference''
input~\cite{bach2015pixel,DBLP:journals/corr/KindermansSMD16,ribeiro2016should,lundberg2017unified,shrikumar2016not}. \citet{ribeiro2016should}
proposed the LIME, methods which randomly draws instances from a
density centered at the sample to be explained, and fits a sparse
linear model to predict the model outputs for these instances.
\citet{shrikumar2016not} presented DeepLIFT, a method designed
specifically for neural networks, which decomposes the output of a
neural network on a specific input by backpropagating the contribution
back to every feature of the input. \citet{lundberg2017unified} used
Shapley values to quantify the importance of features of a given
input, and proposed a sampling based method ``kernel SHAP'' for
approximating Shapley values. Essentially, the two directions both
approximate the model locally via an additive model, with different
definitions of locality. While the first one considers infinitesimal
regions on the decision surface and takes the first-order term in the
Taylor expansion as the additive model, the second one considers the
finite difference between an input vector and a reference vector.

In this paper, our approach to instancewise feature selection is via
mutual information, a conceptually different perspective from existing
approaches. We define an ``explainer,'' or instancewise feature
selector, as a model which returns a distribution over the subset of
features given the input vector.
% The explainer can alternatively be defined as a deterministic
%function that maps each input vector to a corresponding subset of
%features, a special case when the distribution is degenerate.  The
%right explanation, or a correctly selected subset of features, should
%contain the most information of the response variable.
For a given instance, an ideal explainer should assign the highest
probability to the subset of features that are most informative for
the associated model response. This motivates us to maximize the
mutual information between the selected subset of features and the
response variable with respect to the instancewise feature
selector. Direct estimation of mutual information and discrete feature
subset sampling are intractable; accordingly, we derive a tractable
method by first applying a variational lower bound for mutual
information, and then developing a continuous reparametrization of the
sampling distribution.

\begin{table}[t]
\centering
\resizebox{0.9\linewidth}{!}{
\begin{tabular}{c|c|c|c|c|}
& Training & Efficiency & Additive & Model-agnostic\\ 
  LIME \cite{ribeiro2016should} & No & Low & Yes & Yes\\ 
  Kernel SHAP \cite{lundberg2017unified} & No & Low & Yes & Yes\\
  DeepLIFT \cite{shrikumar2016not} & No & High& Yes & No\\ 
  Salient map \cite{simonyan2013deep}& No & High& Yes & No\\ 
  Parzen \cite{baehrens2010explain} & Yes & High & Yes & Yes\\ 
  % Neural-Gradient \cite{simonyan2013deep} & No & High &Yes & No \\ 
  LRP \cite{bach2015pixel} & No & High &Yes & No \\
  L2X & Yes & High & No & Yes
\end{tabular}
} 
\caption{Summary of the properties of different
  methods. ``Training'' indicates whether a method
  requires training on an unlabeled data set. ``Efficiency''
  qualitatively evaluates the computational time during single
  interpretation. ``Additive'' indicates whether a method is locally
  additive. ``Model-agnostic'' indicates whether a method is generic
  to black-box models.}
\label{tab:real-data-summary}
\end{table}

At a high level, the primary differences between our approach and past
work are the following.  First, our framework \textit{globally} learns
a \textit{local} explainer, and therefore takes the distribution of
inputs into consideration. Second, our framework removes the
constraint of local feature additivity on an explainer. These
distinctions enable our framework to yield a more efficient, flexible,
and natural approach for instancewise feature selection. In summary,
our contributions in this work are as follows (see also
Table~\ref{tab:real-data-summary} for systematic comparisons):
\begin{itemize}[noitemsep,nolistsep,leftmargin=0.05\linewidth,topsep=0pt]
  \item We propose an information-based framework for
          instancewise feature selection.
  \item We introduce a learning-based method for instancewise
          feature selection, which is both efficient and model-agnostic.
  % \item Propose two metrics for evaluating model explanation methods. 
\end{itemize}  
Furthermore, we show that the effectiveness of our method on a variety of synthetic and real data sets using both quantitative metric and human evaluation on Amazon Mechanical
Turk.

\section{A framework} 
%\vspace{-1mm}
%------------------------------------------------------------------------

We now lay out the primary ingredients of our general approach.  While
our framework is generic and can be applied to both classification and
regression models, the current discussion is restricted to
classification models.  We assume one has access to the output of a
model as a conditional distribution, $\ProbModel(\cdot \mid x)$, of
the response variable $Y$ given the realization of the input random
variable $X=x \in \real^d$.

\begin{figure}[h]
\centering
\includegraphics[width=0.5\linewidth]{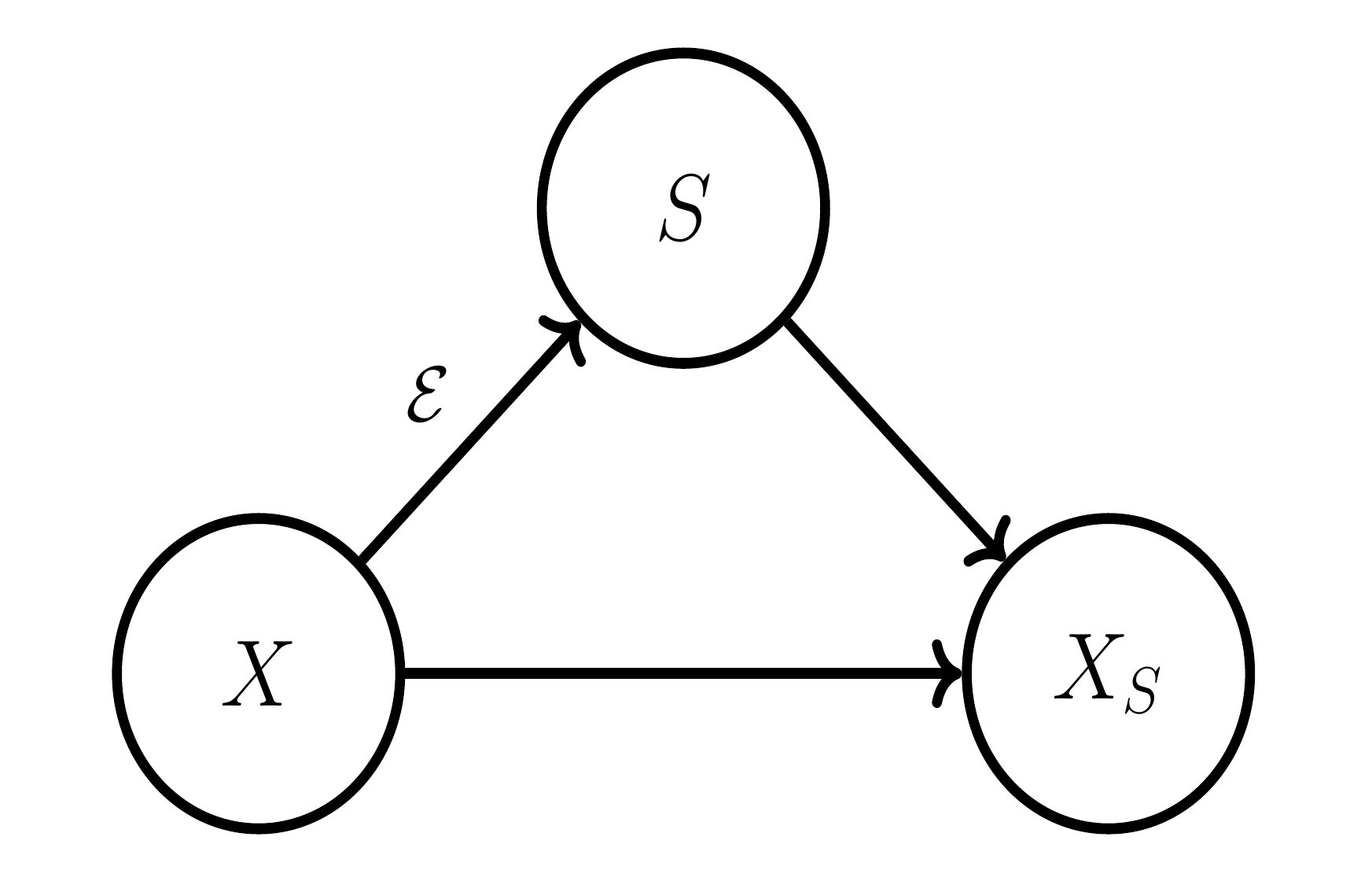}
  \caption{The graphical model of obtaining $X_S$ from $X$.}
\label{fig:graphical_model}
%\vspace{-1mm}
\end{figure}
%\vspace{-2mm}
\subsection{Mutual information}
%\vspace{-1mm}

Our method is derived from considering the mutual information between
a particular pair of random vectors, so we begin by providing some
basic background.  Given two random vectors $X$ and $Y$, the
\textit{mutual information} $I(X;Y)$ is a measure of dependence
between them; intuitively, it corresponds to how much knowledge of one
random vector reduces the uncertainty about the other. More precisely,
the mutual information is given by the Kullback-Leibler divergence of
the product of marginal distributions of $X$ and $Y$ from the joint
distribution of $X$ and $Y$~\cite{cover2012elements}; it takes the
form
\begin{align*}
I(X;Y) = \Exp_{X,Y} \left [ \log \frac{p_{XY}(X,Y)}{p_X(X)p_Y(Y)}
  \right],
\end{align*}
where $p_{XY}$ and $p_X,p_Y$ are the joint and marginal probability
densities if $X,Y$ are continuous, or the joint and marginal
probability mass functions if they are discrete. The expectation is
taken with respect to the joint distribution of $X$ and $Y$. One can
show the mutual information is nonnegative and symmetric in two random
variables. The mutual information has been a popular criteria in
feature selection, where one selects the subset of features that
approximately maximizes the mutual information between the response
variable and the selected
features~\cite{gao2016variational,peng2005feature}. Here we propose to
use mutual information as a criteria for instancewise feature
selection.

%%%%%%%%%%%%%%%%%%%%%%%%%%%%%%%%%%%%%%%%%%%%%%%%%%%%%%%%%%%%%%%%%%%%%%%%%%%%
%\vspace{-2mm}
\subsection{How to construct explanations}
%\vspace{-1mm}
We now describe how to construct explanations using mutual
information.  In our specific setting, the pair $(X, Y)$ are
characterized by the marginal distribution \mbox{$X \sim
  \ProbX(\cdot)$,} and a family of conditional distributions of the
form \mbox{$(Y \mid x) \sim \ProbModel(\cdot \mid x)$.}  For a given
positive integer $k$, let $\powerset_k = \{S\subset 2^{d} \, \mid \,
|S|=k\}$ be the set of all subsets of size $k$.  An \emph{explainer}
$\Explain$ of size $k$ is a mapping from the feature space $\real^d$
to the power set $\powerset_k$; we allow the mapping to be randomized,
meaning that we can also think of $\Explain$ as mapping $x$ to a
conditional distribution $\Prob( S \mid x)$ over $S \in \powerset_k$.
Given the chosen subset $S = \Explain(x)$, we use $x_{S}$ to denote
the sub-vector formed by the chosen features.  We view the choice of
the number of explaining features $k$ as best left in the hands of the
user, but it can also be tuned as a hyper-parameter.

We have thus defined a new random vector $X_S \in \real^{k}$; see
Figure~\ref{fig:graphical_model} for a probabilistic graphical model
representing its construction.  We formulate instancewise feature
selection as seeking explainer that optimizes the criterion
\begin{align}
\label{opt:mi}   
  \max_{\Explain} I(X_S;Y) \quad \text{subject to} \qquad S \sim
  \Explain(X).
\end{align}
In words, we aim to maximize the mutual information between the
response variable from the model and the selected features, as a
function of the choice of selection rule.

It turns out that a global optimum of Problem~\eqref{opt:mi} has a
natural information-theoretic interpretation: it corresponds to the
minimization of the expected length of encoded message for the model
$\ProbModel(Y \mid x)$ using $\ProbModel(Y|x_S)$, where the latter
corresponds to the conditional distribution of $Y$ upon observing the
selected sub-vector.  More concretely, we have the following:

\begin{theorem}
  \label{thm:instance} 
Letting $\ExpModel[\cdot \mid x]$ denote the expectation over
$\ProbModel( \cdot \mid x)$, define
\begin{align}
\Explain^*(x) & \defn \arg \min_{S} \; \ExpModel \left[
  \log\frac{1}{\ProbModel(Y \mid x_S)} \; \Big| \; x \right].
\end{align}
Then $\Explain^*$ is a global optimum of
Problem~\eqref{opt:mi}. Conversely, any global optimum of
Problem~\eqref{opt:mi} degenerates to $\Explain^*$ almost surely over
the marginal distribution $\ProbX$.
\end{theorem} 
The proof of Theorem~\ref{thm:instance} is left to Appendix.
% Appendix~\ref{app:proof}. 
In practice, the above global optimum is
obtained only if the explanation family $\mathcal E$ is sufficiently
large. In the case when $\ProbModel(Y|x_S)$ is unknown or
computationally expensive to estimate accurately, we can choose to
restrict $\mathcal E$ to suitably controlled families so as to prevent
overfitting.

%------------------------------------------------------------------------ 
%\vspace{-2mm}
\section{Proposed method} 
%\vspace{-1mm}
%------------------------------------------------------------------------

A direct solution to Problem~\eqref{opt:mi} is not possible, so that
we need to approach it by a variational approximation.  In particular,
we derive a lower bound on the mutual information, and we approximate
the model conditional distribution $\ProbModel$ by a suitably rich
family of functions.

%------------------------------------------------------------------------

%\vspace{-2mm}
\subsection{Obtaining a tractable variational formulation}
%\vspace{-1mm}
We now describe the steps taken to obtain a tractable variational
formulation.

\paragraph{A variational lower bound:}

Mutual information between $X_S$ and $Y$ can be expressed in terms of
the conditional distribution of $Y$ given $X_S$:
\begin{align*}
I(X_S,Y) &= \Exp \Big[ \log \frac{\ProbModel(X_S,Y)}{\Prob(X_S)
    \ProbModel(Y)} \Big] = \Exp \Big[ \log \frac{\ProbModel(Y |
    X_S)}{\ProbModel(Y)} \Big] \\
& = \Exp \Big[\log \ProbModel(Y|X_S)\Big ] +\text{Const.} \\
& = \Exp_X \Exp_{S|X} \Exp_{Y|X_S} \Big [\log \ProbModel(Y|X_S) \Big]
+ \text{Const.}
\end{align*}
For a generic model, it is impossible to compute expectations under
the conditional distribution $\ProbModel(\cdot \mid x_s)$.  Hence we
introduce a variational family for approximation:
\begin{align}
\Qfamily & \defn \Big \{\Qprob \mid \Qprob = \{x_S\to \Qprob_S(Y|x_S), S \in
\powerset_k\} \Big \}.
\end{align}
Note each member $\Qprob$ of the family $\Qfamily$ is a collection of
conditional distributions $\Qprob_S(Y|x_S)$, one for each choice of
$k$-sized feature subset $S$. For any $\Qprob$, an application of
Jensen's inequality yields the lower bound
\begin{align*}
\Exp_{Y|X_S}[\log \ProbModel(Y | X_S)] & \geq \int \ProbModel(Y
| X_S) \log \Qprob_S(Y | X_S) \\
& = \Exp_{Y|X_S}[\log \Qprob_S(Y|X_S)],
\end{align*}
where equality holds if and only if $\ProbModel(Y \mid X_S)$ and
$\Qprob_S(Y|X_S)$ are equal in distribution. We have thus obtained a
variational lower bound of the mutual information $I(X_S;Y)$.
Problem~\eqref{opt:mi} can thus be relaxed as maximizing the
variational lower bound, over both the explanation $\Explain$ and the
conditional distribution $\Qprob$:
\begin{align}
\label{opt:variational}
\max_{\Explain, \Qprob } \Exp \Big[\log \Qprob_S(Y \mid X_S) \Big]
\qquad \mbox{such that $S \sim \mathcal \Explain(X)$.}
\end{align}
For generic choices $\Qprob$ and $\Explain$, it is still difficult to
solve the variational approximation~\eqref{opt:variational}.  In order
to obtain a tractable method, we need to restrict both $\Qprob$ and $\Explain$ to suitable families over which it is efficient to perform
optimization.

\paragraph{A single neural network for parametrizing $\Qprob$:} Recall that $\Qprob = \{ \Qprob_S( \cdot \mid x_S), \; S \in
\powerset_k\}$ is a collection of conditional distributions with
cardinality $|\Qprob|= {d \choose k}$. We assume $X$ is a continuous
random vector, and $\ProbModel(Y \mid x)$ is continuous with respect
to $x$. Then we introduce a single neural network function $g_\alpha:
\real^d \to \Delta_{c-1}$ for parametrizing $\Qprob$, where the
codomain is a $(c-1)$-simplex \mbox{$\Delta_{c-1}=\{y\in [0,1]^c:0\leq
  y_i\leq 1, \sum_{i=1}^c y_i = 1\}$} for the class distribution, and
$\alpha$ denotes the learnable parameters. We define $\Qprob_S(Y|x_S)
\defn g_\alpha(\tilde x_S)$, where $\tilde x_S \in \real^d$ is
transformed from $x$ with entries not in $S$ replaced by zeros:
\begin{align*}
(\Tilde x_S)_i =
\begin{cases}
x_i, i\in S,\\ 0, i\notin S.
\end{cases}  
\end{align*}
% \red{why do we need $\tilde{X}$. it is not used anywhere else.} 
When $X$ contains discrete features, we embed each discrete feature
with a vector, and the vector representing a specific feature is set
to zero simultaneously when the corresponding feature is not in $S$.

%------------------------------------------------------------------------ 
%\vspace{-2mm}
\subsection{Continuous relaxation of subset sampling} 
%\vspace{-1mm}
%------------------------------------------------------------------------

Direct estimation of the objective function in
equation~\eqref{opt:variational} requires summing over ${d \choose k}$
combinations of feature subsets after the variational
approximation. Several tricks exist for tackling this issue, like
REINFORCE-type Algorithms~\cite{williams1992simple}, or weighted sum
of features parametrized by deterministic functions of $X$. (A similar
concept to the second trick is the ``soft attention'' structure in
vision~\cite{ba2014multiple} and NLP~\cite{bahdanau+al-2014-nmt} where
the weight of each feature is parametrized by a function of the
respective feature itself.) We employ an alternative approach
generalized from Concrete Relaxation (Gumbel-softmax
trick)~\cite{jang2017categorical,maddison2014sampling,maddison2016concrete},
which empirically has a lower variance than REINFORCE and encourages
discreteness~\cite{raffel2017online}.

The Gumbel-softmax trick uses the concrete distribution as a
continuous differentiable approximation to a categorical distribution.
In particular, suppose we want to approximate a categorical random
variable represented as a one-hot vector in $\mathbb R^d$ with
category probability $p_1,p_2,\dots,p_d$. The random perturbation for
each category is independently generated from a Gumbel$(0,1)$
distribution:
\begin{align*}
G_i = -\log (-\log u_i), u_i\sim \text{Uniform}(0,1).
\end{align*} 
We add the random perturbation to the log probability of each category
and take a temperature-dependent softmax over the $d$-dimensional
vector:
\begin{align*}
C_i = \frac{\exp\{(\log p_i + G_i)/\tau\}}{\sum_{j=1}^d \exp\{(\log
  p_j + G_j)/\tau\}}.
\end{align*}
The resulting random vector $C=(C_1,\dots,C_d)$ is called a Concrete
random vector, which we denote by
\begin{align*}
C\sim \text{Concrete}(\log p_1,\dots, \log p_d).
\end{align*}

We apply the Gumbel-softmax trick to approximate weighted subset
sampling. We would like to sample a subset $S$ of $k$ distinct
features out of the $d$ dimensions. The sampling scheme for $S$ can be
equivalently viewed as sampling a $k$-hot random vector $Z$ from
$D^d_k \defn \{z\in\{0,1\}^d \mid \sum z_i = k\}$, with each entry of
$z$ being one if it is in the selected subset $S$ and being zero
otherwise. An importance score which depends on the input vector is
assigned for each feature. Concretely, we define $w_\theta \colon
\real^d \to \real^d$ that maps the input to a $d$-dimensional vector,
with the $i$th entry of $w_\theta(X)$ representing the importance
score of the $i$th feature.

We start with approximating sampling $k$ distinct features out of $d$
features by the sampling scheme below: Sample a single feature out of
$d$ features independently for $k$ times. Discard the overlapping
features and keep the rest. Such a scheme samples at most $k$
features, and is easier to approximate by a continuous relaxation.  We
further approximate the above scheme by independently sampling $k$
independent Concrete random vectors, and then we define a
$d$-dimensional random vector $V$ that is the elementwise maximum of
$C^1,C^2,\dots,C^k$:
\begin{align*}
C^j & \sim \text{Concrete}(w_\theta(X))\text{ i.i.d. for }
j=1,2,\dots,k,\\
V & = (V_1,V_2,\dots,V_d),\quad V_i = \max_{j} C_i^j.
\end{align*}  
The random vector $V$ is then used to approximate the $k$-hot random
vector $Z$ during training.

We write $V = V(\theta,\zeta)$ as $V$ is a function of $\theta$ and a
collection of auxiliary random variables $\zeta$ sampled independently
from the Gumbel distribution. Then we use the elementwise product
$V(\theta,\zeta)\odot X$ between~$V$~and~$X$ as an approximation of
$\tilde X_S$.

\subsection{The final objective and its optimization} 
%\vspace{-1mm}
%------------------------------------------------------------------------

After having applied the continuous approximation of feature subset
sampling, we have reduced Problem~\eqref{opt:variational} to the
following:
\begin{align}
  \label{opt:parametrized}
\max_{ \theta,\alpha} \Exp_{X,Y,\zeta} \Big[ \log
  g_\alpha(V(\theta,\zeta)\odot X, Y) \Big],
\end{align}
where $g_\alpha$ denotes the neural network used to approximate the
model conditional distribution, and the quantity $\theta$ is used to
parametrize the explainer. In the case of classification with $c$
classes, we can write
\begin{align}
 \label{eq:final_obj}  
\Exp_{X,\zeta} \Big[ \sum_{y=1}^c [\ProbModel(y \mid X) \log
    g_\alpha(V(\theta,\zeta) \odot X, y) \Big].
\end{align}
Note that the expectation operator $\Exp_{X, \zeta}$ does not depend
on the parameters $(\alpha, \theta)$, so that during the training
stage, we can apply stochastic gradient methods to jointly optimize
the pair $(\alpha, \theta)$.  In each update, we sample a mini-batch
of unlabeled data with their class distributions from the model to be
explained, and the auxiliary random variables $\zeta$, and we then
compute a Monte Carlo estimate of the gradient of the objective
function~\eqref{eq:final_obj}.

% \mjwcomment{Size of mini-batch?  Other details on how exactly you optimzied.
%   There needs to be enough detail that the experiments are in principle
%   reproducible.  Or you need to link to a github page with the code.}
%------------------------------------------------------------------------ 
%\vspace{-2mm}
\subsection{The explaining stage} 
%\vspace{-1mm}
%------------------------------------------------------------------------

During the explaining stage, the learned explainer maps each sample
$X$ to a weight vector $w_\theta(X)$ of dimension $d$, each entry
representing the importance of the corresponding feature for the
specific sample $X$. In order to provide a deterministic explanation
for a given sample, we rank features according to the weight vector,
and the $k$ features with the largest weights are picked as the
explaining features.
 
For each sample, only a single forward pass through the neural network
parametrizing the explainer is required to yield explanation. Thus our
algorithm is much more efficient in the explaining stage compared to
other model-agnostic explainers like LIME or Kernel SHAP which require
thousands of evaluations of the original model per sample.

%------------------------------------------------------------------------ 
\vspace{-2mm}
\section{Experiments} 
\vspace{-1mm}
%------------------------------------------------------------------------

We carry out experiments on both synthetic and real data sets. 
For all experiments, we use RMSprop~\cite{maddison2016concrete} with the default hyperparameters for
optimization. We also fix the step size to be $0.001$ across
experiments. The temperature for Gumbel-softmax approximation is fixed
to be $0.1$. Codes for reproducing the key results are available online at \url{https://github.com/Jianbo-Lab/L2X}. 

\begin{figure}[!bt] 
%\vspace{-3mm}
\centering 
\includegraphics[width=0.9\linewidth]{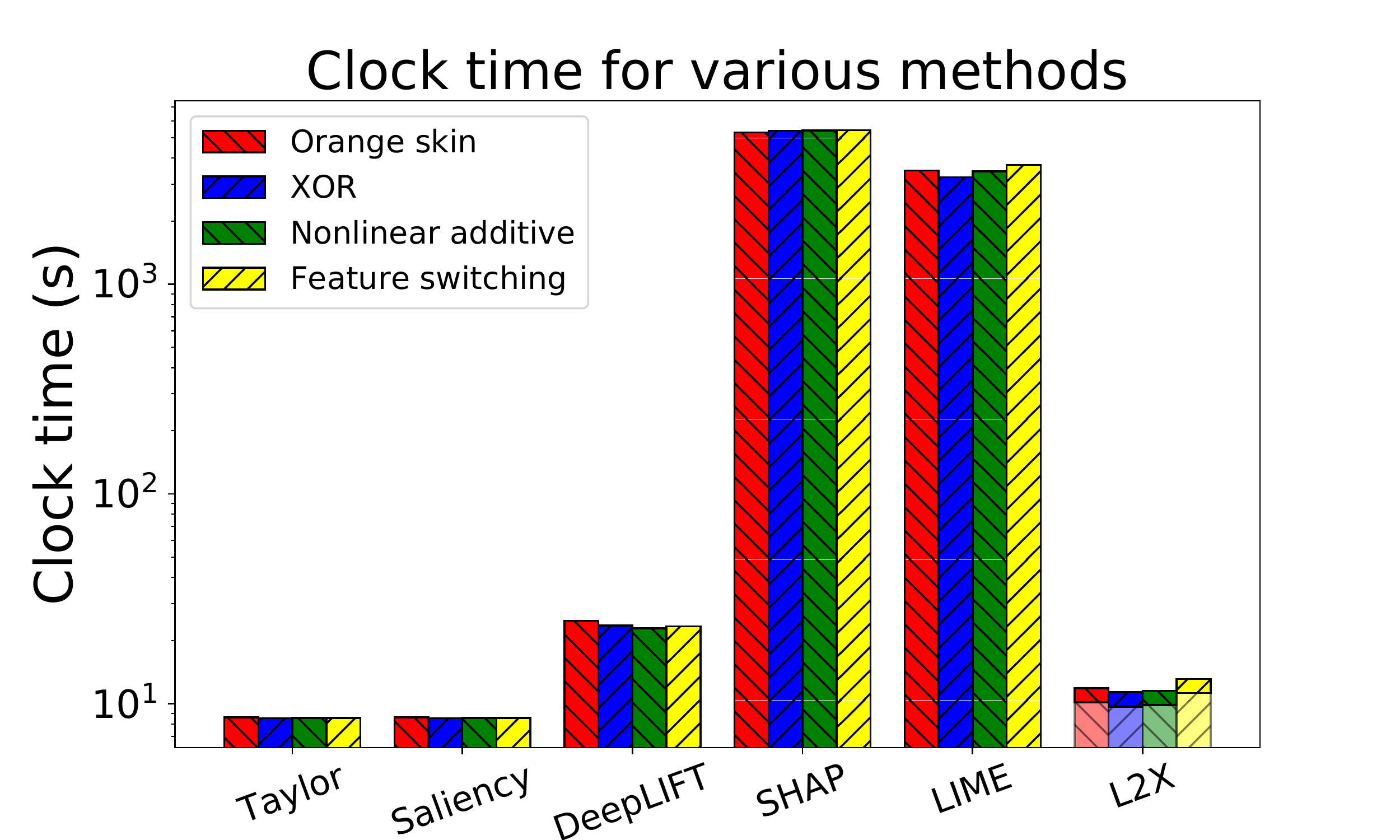}
\caption{The clock time (in log scale) of explaining $10,000$ samples
  for each method. The training time of L2X is shown in translucent
  bars.}
\label{fig:time}  
\end{figure}
%\red{need to say something about the overall dimension of each problem. 10? 100?}
% \subsection{Simulated Rules}
% In this section, we carry experiments on two simulated datasets. We first generate a ten-dimensional random vector $X\in \mathbb R^{10}$ from standard Gaussian. Then we generate the response variable $Y$ are generated according to the following two rules respectively for two datasets:
% \begin{itemize}
%   \item If $X_5>0$, then $Y=1$ if and only if $X_6+X_7>0$. If $X_5\leq 0$, then $Y=1$ if and only if $X_8+X_9>0$.
%   \item Let $Z = \max\{X_4,X_5\} + \max\{X_6,X_7\}$. Let $Y=1$ if and only if $Z>0$.
% \end{itemize}
% In the first dataset, the variables $X_6,X_7$ are important if $X_5>0$. Othervise $X_8,X_9$ are important. In the second dataset, the maximum of $X_4,X_5$, and the maximum of $X_6,X_7$ are important. In both cases, our method can capture these important variables. 

%------------------------------------------------------------------------ 
%\vspace{-2mm}
\subsection{Synthetic Data} 
%\vspace{-1mm}
%------------------------------------------------------------------------
We begin with experiments on four synthetic data sets: 
% (\emph{i}) $2$-dimensional XOR as binary classification. The input vector $X$ is generated from a $10$-dimensional standard Gaussian. The response variable $Y$ is generated from $P(Y=1|X)\propto \exp \{X_1X_2\}$; 
% (\emph{ii}) Orange Skin. The input vector $X$ is generated from a $10$-dimensional standard Gaussian. The response variable $Y$ is generated from $P(Y=1|X)\propto \exp \{\sum_{i=1}^4 X_i^2 - 4\}$; 
% (\emph{iii}) Nonlinear additive model. Generate $X$ from a 10-dimensional standard Gaussian. The response variable $Y$ is generated from $P(Y=1|X) \propto \exp\{-100\sin(2X_1)+2|X_2| + X_3 + \exp\{-X_4\}\}$;  
% (\emph{iv}) Switch feature. Generate $X_{1}$ from a mixture of two Gaussians centered at $\pm 3$ respectively with equal probability. If $X_{1}$ is generated from the Gaussian centered at $3$, the $2-5$th dimensions are used to generate $Y$ like the orange skin model. Otherwise, the $6-9th$ dimensions are used to generate $Y$ from the nonlinear additive model.   

\begin{itemize}[noitemsep,nolistsep,leftmargin=0.05\linewidth,topsep=0pt] 
\item $2$-dimensional XOR as binary classification. The input vector $X$ is generated from a $10$-dimensional standard Gaussian. The response variable $Y$ is generated from $P(Y=1|X)\propto \exp \{X_1X_2\}$. 
% \begin{small}
% \begin{align*}
% P(Y=1|X)\propto \exp \{X_1X_2\}.
% \end{align*}
% \end{small}
\item Orange Skin. The input vector $X$ is generated from a $10$-dimensional standard Gaussian. The response variable $Y$ is generated from $P(Y=1|X)\propto \exp \{\sum_{i=1}^4 X_i^2 - 4\}$. 
% \begin{small}
% \begin{align*}
% P(Y=1|X)\propto \exp \{\sum_{i=1}^4 X_i^2 - 4\}.
% \end{align*}
% \end{small}
 % Consider the 8 corners of the 3-dimensional hypercube $(v_1, v_2, v_3) \in \{-1,1\}^3$, and group them by the tuples $(v_1 v_3, v_2 v_3)$, leaving 4 sets of vectors paired with their negations $\{v^{(i)},-v^{(i)}\}$. In each class $i$, samples are generated from the mixture of standard Gaussian centered at $v_{(i)}$ and $-v_{(i)}$ with equal probability. Each sample additionally has $7$ standard normal noise features for a total of $10$ dimensions. 
% \item Orange Skin \cite{friedman2001elements}. Given $Y=-1$, ten features $(X_1,X_2,\dots,X_{10})$ are independent Gaussian random variables. Given $Y=1$, the first four features are independent Guassian conditioned on $9\leq \sum_{j=1}^4 X_j^2\leq 16$ and the remaining six features are independent Gaussian. 
\item Nonlinear additive model. Generate $X$ from a 10-dimensional standard Gaussian. The response variable $Y$ is generated from $P(Y=1|X) \propto \exp\{-100\sin(2X_1)+2|X_2| + X_3 + \exp\{-X_4\}\}$. 
% \begin{small}
% \begin{align*}
% P(Y=1|X) \propto \exp\{-100\sin(2X_1)+2|X_2| + X_3 + \exp\{-X_4\}\}.
% \end{align*}
% \end{small}
\item Switch feature. Generate $X_{1}$ from a mixture of two Gaussians centered at $\pm 3$ respectively with equal probability. If $X_{1}$ is generated from the Gaussian centered at $3$, the $2-5$th dimensions are used to generate~$Y$ like the orange skin model. Otherwise, the $6-9th$ dimensions are used to generate $Y$ from the nonlinear additive model.   
\end{itemize}

\begin{figure*}[h] 
\centering 
\includegraphics[width=0.4\linewidth]{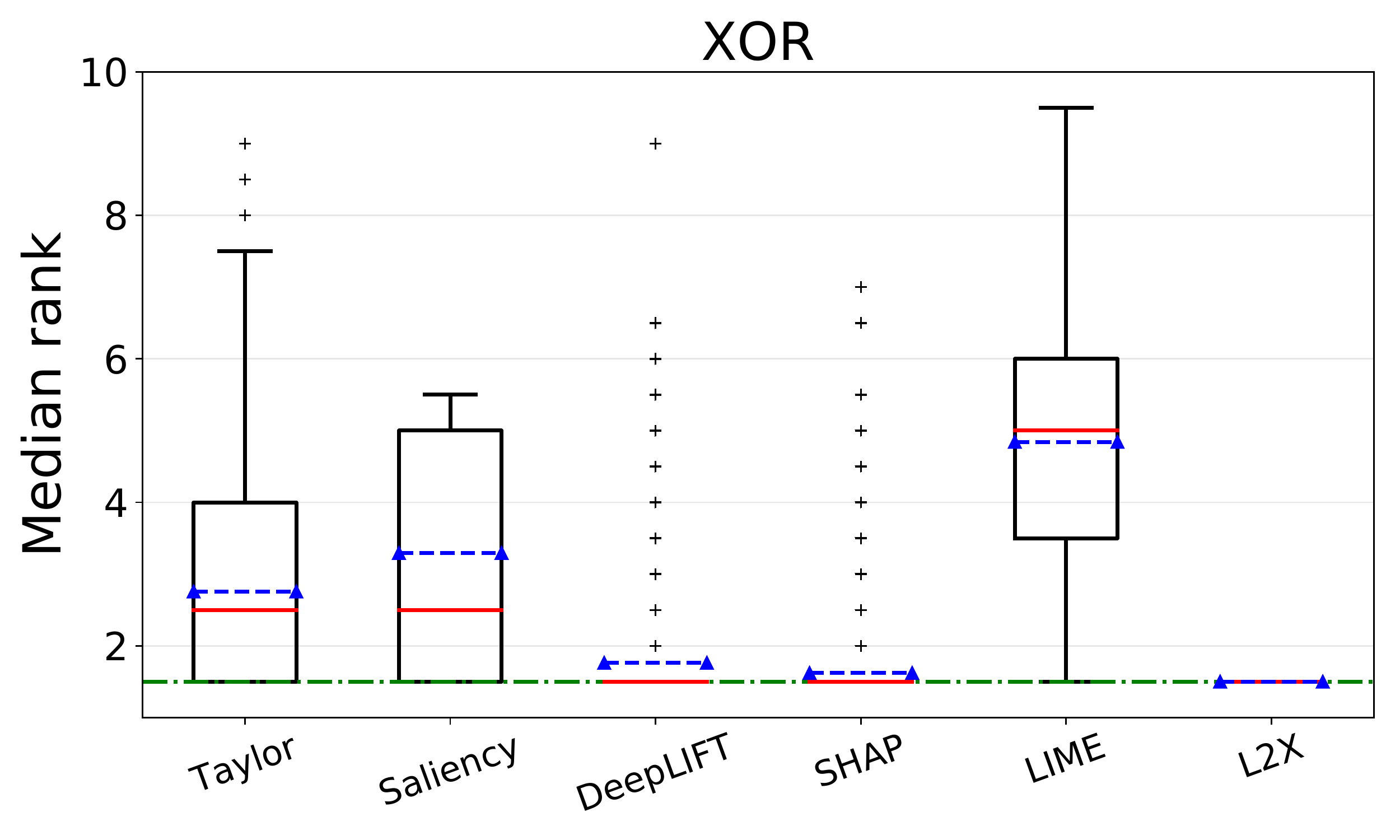}%
\includegraphics[width=0.4\linewidth]{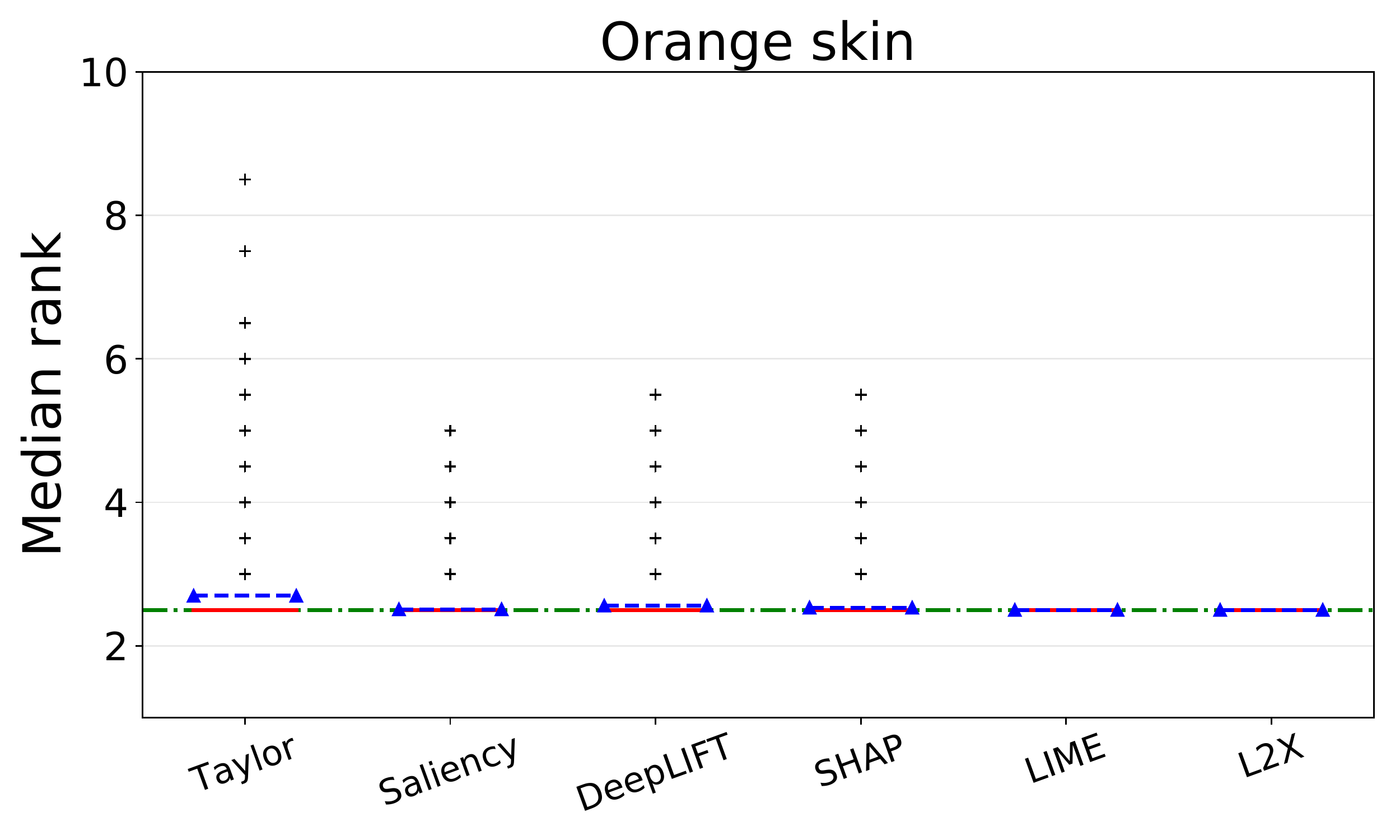}
\includegraphics[width=0.4\linewidth]{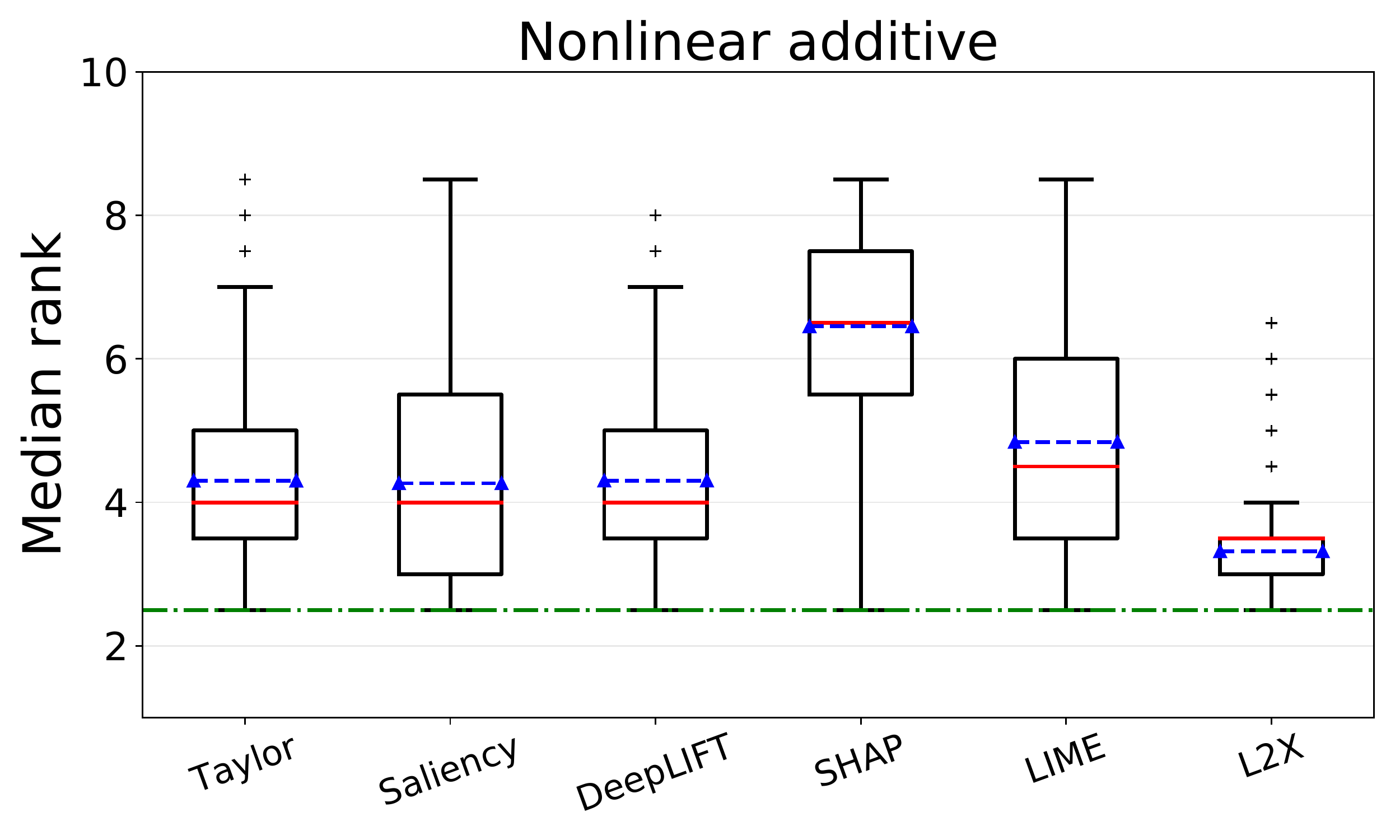}%
\includegraphics[width=0.4\linewidth]{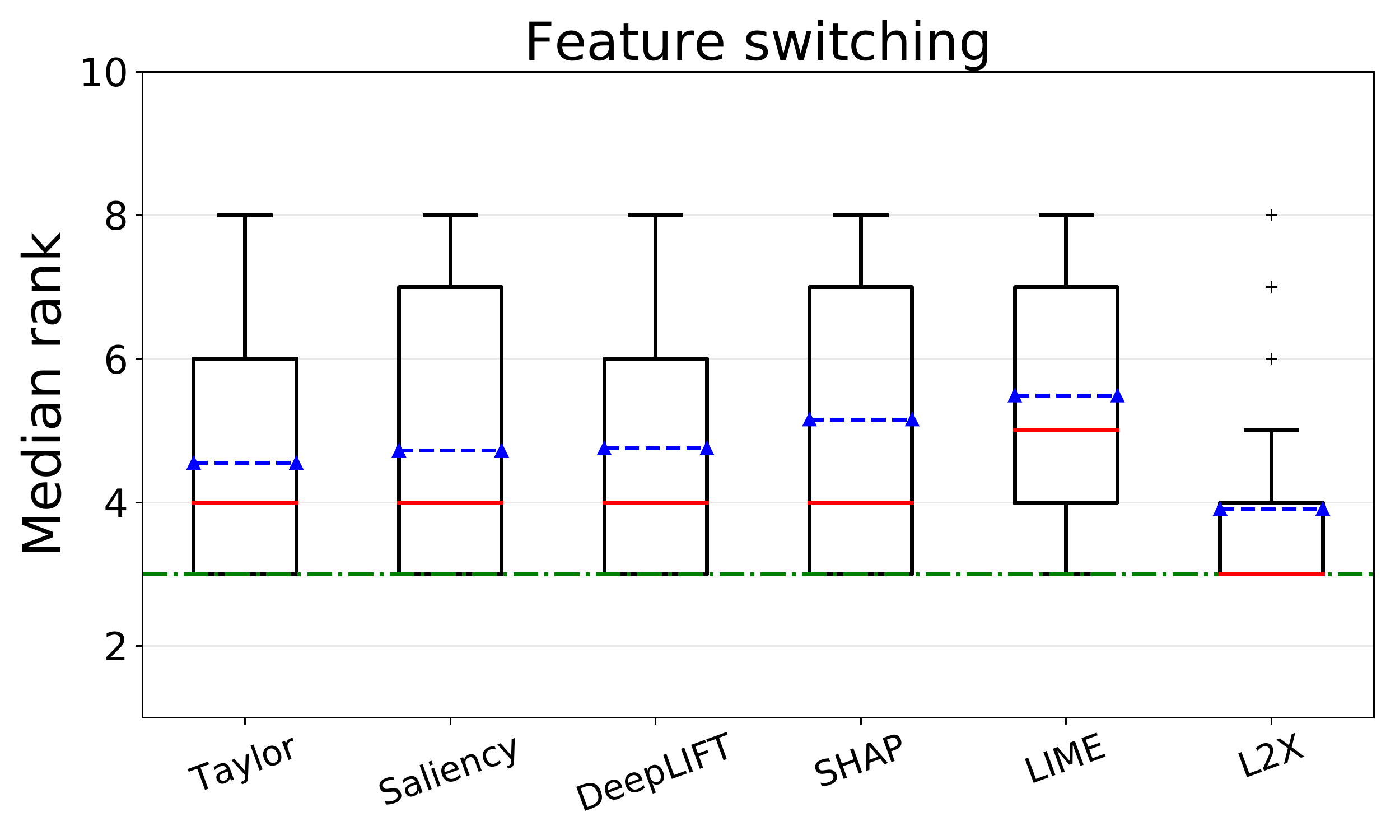} 
% \includegraphics[width=0.33\linewidth]{results-synthetic-regression-5} 
%\vspace{-3mm}
\caption{The box plots for the median ranks of the influential features by each sample, over $10,000$ samples for each data set. The red line and the dotted blue line on each box is the median and the mean respectively. Lower median ranks are better. The dotted green lines indicate the optimal median rank.}
\label{fig:synthetic_boxplot}  
\end{figure*}

\begin{table*}[ht!]
\centering
\resizebox{.8\textwidth}{!}{
 \begin{tabularx}{\textwidth}{||c |c| X||} 
 \hline
 Truth & Model & {Key words}  \\ [0.5ex] %Key words &&&&&&&&&
 \hline\hline
 % 1 & 6 & 87837 & 787 \\ 
 % 2 & 7 & 78 & 5415 \\
 % 3 & 545 & 778 & 7507 \\
 % 4 & 545 & 18744 & 7560 \\
 % 5 & 88 & 788 & 6344 \\ [1ex] 
 positive& positive& \small{Ray Liotta and Tom Hulce shine in this sterling example of brotherly \hlyellow{love} and commitment. Hulce plays Dominick, (nicky) a \hlyellow{mildly} mentally handicapped young man who is putting his 12 minutes younger, twin brother, Liotta, who plays Eugene, through medical school. It is set in Baltimore and \hlyellow{deals} with the issues of sibling rivalry, the unbreakable \hlyellow{bond} of twins, child abuse and good always \hlyellow{winning} out over evil. It is \hlyellow{captivating}, and filled with laughter and \hlyellow{tears}. If you have not yet seen this film, \hlyellow{please rent} it, I promise, you'll be amazed at how such a \hlyellow{wonderful} film could go unnoticed.}\\
 \hline
negative & negative & \small{\hlyellow{Sorry} to go against the flow
  but I thought \hlyellow{this} film was \hlyellow{unrealistic},
  \hlyellow{boring} and way too long. I got \hlyellow{tired} of
  watching Gena Rowlands long arduous battle with herself and the
  crisis she was experiencing. Maybe the film has some cinematic value
  or represented an important \hlyellow{step} for the director but
  \hlyellow{for} pure \hlyellow{entertainment value}. I wish I would
  \hlyellow{have} skipped it.}\\
\hline negative&positive& \small{This movie is \hlyellow{chilling
    reminder} of Bollywood being just a parasite of
  Hollywood. Bollywood also tends to feed on past blockbusters for
  furthering its industry. Vidhu Vinod Chopra made this
  \hlyellow{movie} with the reasoning that a cocktail mix of deewar
  and on the waterfront will bring home an \hlyellow{oscar}. It turned
  out to be rookie mistake\hlyellow{.} Even the \hlyellow{idea} of the
  title is \hlyellow{inspired} from the Elia Kazan
  \hlyellow{classic}. In the original, Brando is shown as raising
  doves as symbolism of peace. \hlyellow{Bollywood} \hlyellow{must}
  move out of Hollywoods shadow if it needs to be taken seriously.}
\\ 
\hline positive& negative& \small{When a small town is threatened by a
  child killer, a lady \hlyellow{police} officer goes after him by
  pretending to be his friend. As she becomes more and more
  emotionally \hlyellow{involved} with the murderer her psyche begins
  to take a beating causing her to lose focus on the \hlyellow{job} of
  catching the \hlyellow{criminal}. Not a film of high voltage
  excitement, but \hlyellow{solid police} work and \hlyellow{a good
    depiction} of the faulty mind of a psychotic \hlyellow{loser}.}\\
% positive& positive& This movie \hlyellow{was} on british tv last night, and is \hlyellow{wonderful}! Strong \hlyellow{women}, \hlyellow{great} music (most of the time) and just makes \hlyellow{you think}. We do have stereotypes of what older people ought to do, and there are \hlyellow{fantastic cameos} of the sensible but worried children. Getting near to my \hlyellow{best} movie \hlyellow{ever} !\\
\hline
 \end{tabularx} 
 } 
 %\vspace{-3mm}
 \caption{True labels and labels predicted by the model are in the first two columns. Key words picked by L2X are highlighted in yellow.}
\label{table:imdb_word}
%\vspace{-3mm}
\end{table*}

The first three data sets are modified from commonly used data sets in the feature selection literature~\cite{chen2017kernel}. The fourth data set is designed specifically for instancewise feature selection. Every sample in the first data set has the first two dimensions as true features, where each dimension itself is independent of the response variable $Y$ but the combination of them has a joint effect on $Y$. In the second data set, the samples with positive labels centered around a sphere in a four-dimensional space. The sufficient statistic is formed by an additive model of the first four features. The response variable in the third data set is generated from a nonlinear additive model using the first four features. The last data set switches important features (roughly) based on the sign of the first feature. The $1-5$ features are true for samples with $X_1$ generated from the Gaussian centered at $-3$, and the $1,6-9$ features are true otherwise.   

We compare our method L2X (for ``Learning to Explain'') with several
strong existing algorithms for instancewise feature selection,
including Saliency~\cite{simonyan2013deep}, DeepLIFT~\cite{shrikumar2016not}, SHAP~\cite{lundberg2017unified}, LIME~\cite{ribeiro2016should}.  Saliency refers to the method that computes
the gradient of the selected class with respect to the input feature
and uses the absolute values as importance scores. SHAP refers to
Kernel SHAP. The number of samples used for explaining each instance
for LIME and SHAP is set as default for all experiments. We also
compare with a method that ranks features by the
input feature times the gradient of the selected class with respect to
the input feature. \citet{shrikumar2016not} showed it is equivalent to
LRP~\cite{bach2015pixel} when activations are piecewise linear, and
used it in \citet{shrikumar2016not} as a strong baseline. We call it
``Taylor'' as it is the first-order Taylor approximation of the model.

Our experimental setup is as follows. For each data set, we train a neural network model with three hidden dense layers. We can safely assume the neural network has successfully captured the important features, and ignored noise features, based on its error rate. Then we use Taylor, Saliency, DeepLIFT, SHAP, LIME, and L2X for instancewise feature selection on the trained neural network models. For L2X, the explainer is a neural network composed of two hidden layers. The variational family is composed of three hidden layers. All layers are linear with dimension $200$. The number of desired features $k$ is set to the number of true features.

The underlying true features are known for each sample, and hence the
median ranks of selected features for each sample in a validation data
set are reported as a performance metric, the box plots of which have
been plotted in Figure~\ref{fig:synthetic_boxplot}. We observe that
L2X outperforms all other methods on nonlinear additive and feature
switching data sets. On the XOR model, DeepLIFT, SHAP and L2X achieve
the best performance. On the orange skin model, all algorithms have
near optimal performance, with L2X and LIME achieving the most stable
performance across samples. 

We also report the clock time of each method in Figure~\ref{fig:time},
where all experiments were performed on a single NVidia Tesla k80 GPU,
coded in TensorFlow. Across all the four data sets, SHAP and LIME are
the least efficient as they require multiple evaluations of the
model. DeepLIFT, Taylor and Saliency requires a backward pass of the
model. DeepLIFT is the slowest among the three, probably due to the
fact that backpropagation of gradients for Taylor and Saliency are
built-in operations of TensorFlow, while backpropagation in DeepLIFT
is implemented with high-level operations in TensorFlow. Our method
L2X is the most efficient in the explanation stage as it only requires
a forward pass of the subset sampler. It is much more efficient
compared to SHAP and LIME even after the training time has been taken
into consideration, when a moderate number of samples (10,000) need to
be explained. As the scale of the data to be explained increases, the
training of L2X accounts for a smaller proportion of the over-all
time. Thus the relative efficiency of L2X to other algorithms
increases with the size of a data set.

\begin{table*}[!bt]
\centering
\resizebox{.8\textwidth}{!}{
 \begin{tabularx}{\textwidth}{||c |c| X||} 
 \hline
 Truth & Predicted & Key sentence  \\ [0.5ex] %Key words &&&&&&&&&
 \hline\hline
 % 1 & 6 & 87837 & 787 \\ 
 % 2 & 7 & 78 & 5415 \\
 % 3 & 545 & 778 & 7507 \\
 % 4 & 545 & 18744 & 7560 \\
 % 5 & 88 & 788 & 6344 \\ [1ex]  
positive& positive& \small{There are few really hilarious films about science fiction but this one will knock your sox off. The lead Martians Jack Nicholson take-off is side-splitting. The plot has a very clever twist that has be seen to be enjoyed. \hlyellow{This is a movie with heart and excellent acting by all.} Make some popcorn and have a great evening.}\\ 
\hline
negative& negative& \small{You get 5 writers together, have each write a different story with a different genre, and then you try to make one movie out of it. Its action, its adventure, its sci-fi, its western, its a mess. \hlyellow{Sorry, but this movie absolutely stinks.} 4.5 is giving it an awefully high rating. That said, its movies like this that make me think I could write movies, and I can barely write.}\\ 
\hline
negative& positive& \small{This movie is not the same as the 1954 version with Judy garland and James mason, and that is a shame because the 1954 version is, in my opinion, much better. I am not denying Barbra Streisand's talent at all. \hlyellow{She is a good actress and brilliant singer.} I am not acquainted with Kris Kristofferson's other work and therefore I can't pass judgment on it. However, this movie leaves much to be desired. It is paced slowly, it has gratuitous nudity and foul language, and can be very difficult to sit through. However, I am not a big fan of rock music, so its only natural that I would like the judy garland version better. See the 1976 film with Barbra and Kris, and judge for yourself.}\\  
\hline
positive& negative& \small{The first time you see the second renaissance it may look boring. Look at it at least twice and definitely watch part 2. it will change your view of the matrix. Are the human people the ones who started the war? \hlyellow{Is ai a bad thing?}}\\
\hline

 \end{tabularx}
 } 
 %\vspace{-3mm}
 \caption{True labels and labels from the model are shown in the first two columns. Key sentences picked by L2X highlighted in yellow.}
%\vspace{-3mm}

\label{table:imdb_sent}

\end{table*}

%------------------------------------------------------------------------ 
%\vspace{-2mm}
\subsection{IMDB}  
%\vspace{-1mm}
%------------------------------------------------------------------------

The Large Movie Review Dataset (IMDB) is a dataset of movie reviews for sentiment classification~\cite{maas2011learning}. It contains $50,000$ labeled movie reviews, with a split of $25,000$ for training and $25,000$ for testing. The average document length is $231$ words, and $10.7$ sentences. 
% Since the testing data is not available to the public, we use $20\%$ of the training data as our validation data set, and the rest $80\%$ for training, with the default train/validation split.
We use L2X to study two popular classes of models for sentiment analysis on the IMDB data set. 

%------------------------------------------------------------------------ 
\subsubsection{Explaining a CNN model with key words} 
%------------------------------------------------------------------------

Convolutional neural networks (CNN) have shown excellent performance for sentiment analysis~\cite{kim2014convolutional,zhang2015sensitivity}. 
% A simple multi-layer CNN is put on top of word vectors pre-trained with an unsupervised language model on a large corpus \cite{mikolov2013distributed}, with tricks of regularization like Dropout \cite{srivastava2014dropout} employed between layers to reduce generalization error. 
We use a simple CNN model on Keras~\cite{chollet2015keras} for the
IMDB data set, which is composed of a word embedding of dimension
$50$, a $1$-D convolutional layer of kernel size $3$ with $250$
filters, a max-pooling layer and a dense layer of dimension $250$ as
hidden layers. Both the convolutional and the dense layers are
followed by ReLU as nonlinearity, and
Dropout~\cite{srivastava2014dropout} as regularization. Each review is
padded/cut to $400$ words. The CNN model achieves $90\%$ accuracy on
the test data, close to the state-of-the-art performance (around
$94\%$). We would like to find out which $k$ words make the most
influence on the decision of the model in a specific review. The
number of key words is fixed to be $k=10$ for all the experiments. 

The explainer of L2X is composed of a global component and a local component (See Figure 2 in \citet{yang2018greedy}). The input is initially fed into a common embedding layer followed by a convolutional layer with $100$ filters. Then the local component processes the common output using two convolutional layers with $50$ filters, and the global component processes the common output using a max-pooling layer followed by a $100$-dimensional dense layer. Then we concatenate the global and local outputs corresponding to each feature, and process them through one convolutional layer with $50$ filters, followed by a Dropout layer \cite{srivastava2014dropout}. Finally a convolutional network with kernel size $1$ is used to yield the output. All previous convolutional layers are of kernel size 3, and ReLU is used as nonlinearity. 
% a \mbox{$50$-dimensional} word embedding,
% three $1$-D convolutional layers of kernel size $3$ and filter size
% $100$. 
The variational family is composed of an word embedding layer
of the same size, followed by an average pooling and a
\mbox{$250$-dimensional} dense layer. Each entry of the output vector $V$
from the explainer is multiplied with the embedding of the respective
word in the variational family. We use both automatic metrics and human annotators to validate the effectiveness of L2X. 

\paragraph{Post-hoc accuracy.} We introduce \textit{post-hoc accuracy} for quantitatively validating the effectiveness of our method. Each model explainer outputs a subset of features $X_S$
for each specific sample $X$. We use $\ProbModel(y|\tilde X_S)$ to
approximate $\ProbModel(y|X_S)$. That is, we feed in the sample $X$ to
the model with unselected words masked by zero paddings. Then we
compute the accuracy of using $\ProbModel(y|\tilde X_S)$ to predict
samples in the test data set labeled by $\ProbModel(y|X)$, which we
call \textit{post-hoc accuracy} as it is computed after instancewise
feature selection.

\paragraph{Human accuracy.} When designing human experiments, we assume
that the key words convey an attitude toward a movie, and can thus be
used by a human to infer the review sentiment. This assumption has
been partially validated given the aligned outcomes provided by
post-hoc accuracy and by human judges, because the alignment implies
the consistency between the sentiment judgement based on selected
words from the original model and that from humans. Based on this
assumption, we ask humans on Amazon Mechanical Turk (AMT) to infer the
sentiment of a review given the ten key words selected by each
explainer. The words adjacent to each other, like ``not good at all,''
keep their adjacency on the AMT interface if they are selected
simultaneously. The reviews from different explainers have been mixed
randomly, and the final sentiment of each review is averaged over the
results of multiple human annotators. We measure whether the labels
from human based on selected words align with the labels provided by
the model, in terms of the average accuracy over $500$ reviews in the
test data set. Some reviews are labeled as ``neutral'' based on
selected words, which is because the selected key words do not contain
sentiment, or the selected key words contain comparable numbers of
positive and negative words. Thus these reviews are neither put in the
positive nor in the negative class when we compute accuracy. We call
this metric \textit{human accuracy}.

% Put confusion matrices in supplementary materials. one method, eight columns. 
% statistical significance of the accuracy.

The result is reported in Table~\ref{table:l2x}. We observe that the model prediction based on only ten words selected by L2X align with the original prediction for over $90\%$ of the data. The human judgement given ten words also aligns with the model prediction for $84.4\%$ of the data. The human accuracy is even higher than that based on the original review, which is $83.3\%$ \cite{yang2018greedy}. This indicates the selected words by L2X can serve as key words for human to understand the model behavior. Table~\ref{table:imdb_word} shows the results of our model on four examples.

% \begin{table}[!bt] 
% \centering
% \resizebox{.49\textwidth}{!}{
%  \begin{tabular}{||c |c c c c c||} 
%  \hline
%  &Taylor&Saliency&SHAP&LIME&L2X \\
%  \hline
%  Post-hoc accuracy & 0.818& 0.621&0.659 &0.736 &\textbf{0.849}\\
% \hline
%  VMI  &0.1022 &0.0984 &0.0120 &0.1465&\textbf{0.287} \\
% \hline
%  Human accuracy& 0.738 & 0.552 & 0.638 &0.608 &\textbf{0.774}\\%[1ex] 
%  \hline
%  \end{tabular}
%  }

%  \caption{Comparison on IMDB for LSTM model.}
% \label{table:result_sent} 
% \end{table} 

% We compare\footnote{We omit DeepLIFT because the TensorFlow version of DeepLIFT code does not support max-pooling, dropout directly.}  L2X with Taylor, Saliency, SHAP, and LIME.  For Taylor,
% we use the inner product between the embedding
% of each word and the gradient of the model with respect to the word
% embedding as the importance score of a word. For Saliency, we use the
% $l_1$ norm of the gradient with respect to the word embedding as its
% importance score, as is proposed in~\citet{li2015visualizing}. 

% The results of comparison measured by the three metrics are shown in
% Table~\ref{table:result_word}. L2X achieves the best performance
% across the three metrics, followed by Taylor.
%SHAP and LIME are less effective for the given task, and Saliency basically does not work. 

%------------------------------------------------------------------------ 
%\vspace{-2mm}
\subsubsection{Explaining hierarchical LSTM}  
%\vspace{-1mm}
%------------------------------------------------------------------------

Another competitive class of models in sentiment analysis uses
hierarchical LSTM~\cite{hochreiter1997long,li2015hierarchical}. We
build a simple hierarchical LSTM by putting one layer of LSTM on top
of word embeddings, which yields a representation vector for each
sentence, and then using another LSTM to encoder all sentence
vectors. The output representation vector by the second LSTM is passed
to the class distribution via a linear layer. Both the two LSTMs and
the word embedding are of dimension $100$. The word embedding is
pretrained on a large corpus~\cite{mikolov2013distributed}. Each
review is padded to contain $15$ sentences. The hierarchical LSTM
model gets around 90\% accuracy on the test data. We take each
sentence as a single feature group, and study which sentence is the
most important in each review for the model.

% We compare L2X with Taylor, Saliency, LIME, and SHAP for this task. For Taylor and Saliency, the gradient with respect to all word embedding vectors in each sentence is computed, and its $L_1$ norm and its inner product with the input vector are used as importance scores respectively. We set the number of samples as default for SHAP. 

 The explainer of L2X is composed of a $100$-dimensional word embedding followed by a convolutional layer and a max pooling layer to encode each sentence. The encoded sentence vectors are fed through three convolutional layers and a dense layer to get sampling weights for each sentence. The variational family also encodes each sentence with a convolutional layer and a max pooling layer. The encoding vectors are weighted by the output of the subset sampler, and passed through an average pooling layer and a dense layer to the class probability. All convolutional layers are of filter size $150$ and kernel size $3$. In this setting, L2X can be interpreted as a hard attention model~\cite{xu2015show} that employs the Gumbel-softmax trick. 

Comparison is carried out with the same metrics. For \textit{human accuracy}, one selected sentence for each review is shown to human annotators. The other experimental setups are kept the same as above. We observe that post-hoc accuracy reaches $84.4\%$ with one sentence selected by L2X, and human judgements using one sentence align with the original model prediction for $77.4\%$ of data. 
% The result is shown in Table~\ref{table:result_sent}. L2X achieves the best performance across the three metrics. 
Table~\ref{table:imdb_sent} shows the explanations from our model on four examples.

% \begin{table}[H] 
% \centering
% \resizebox{.49\textwidth}{!}{
%  \begin{tabular}{||c |c c c c c||} 
%  \hline
%  &Taylor&Saliency&SHAP&LIME&L2X \\
%  \hline
%  Post-hoc accuracy & 0.941 &0.759 &0.631 & 0.581 &\textbf{0.958}\\
% \hline
%  VMI  &0.524 &0.098  &0.029 & 0.013&\textbf{0.564} \\%[1ex]
% % \hline
% %  Human accuracy& 0.738 & 0.552 & 0.638 & &\textbf{0.746}\\
%  \hline  
%  \end{tabular} 
%  }
%  %\vspace{-4mm}
%  \caption{Comparison on MNIST for selecting key patches.}
% \label{table:mnist} 
%  %\vspace{-4mm}
% \end{table} 

\begin{figure}[!bt]
\centering 
\includegraphics[width=0.10\linewidth]{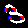}%
\includegraphics[width=0.10\linewidth]{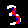}%
\includegraphics[width=0.10\linewidth]{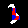}%
\includegraphics[width=0.10\linewidth]{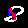}%
\includegraphics[width=0.10\linewidth]{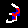}%
\includegraphics[width=0.10\linewidth]{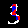}%
\includegraphics[width=0.10\linewidth]{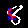}%
\includegraphics[width=0.10\linewidth]{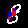}% 
\includegraphics[width=0.10\linewidth]{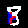}%
\includegraphics[width=0.10\linewidth]{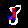} 

\includegraphics[width=0.10\linewidth]{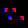}%
\includegraphics[width=0.10\linewidth]{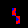}%
\includegraphics[width=0.10\linewidth]{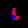}%
\includegraphics[width=0.10\linewidth]{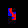}%
\includegraphics[width=0.10\linewidth]{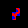}%
\includegraphics[width=0.10\linewidth]{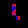}%
\includegraphics[width=0.10\linewidth]{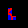}%
\includegraphics[width=0.10\linewidth]{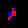}% 
\includegraphics[width=0.10\linewidth]{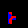}% 
\includegraphics[width=0.10\linewidth]{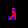}  
\vspace{-6mm}   
\caption{The above figure shows ten randomly selected figures of $3$ and $8$ in the validation set. The first line include the original digits while the second line does not. The selected patches are colored with red if the pixel is activated (white) and blue otherwise.}
\label{fig:mnist}
 % \vspace{-4mm}  
\end{figure}

\begin{table}[!bt] 
\centering 
\resizebox{.49\textwidth}{!}{
 \begin{tabular}{||c |c c c||} 
 \hline
 &IMDB-Word & IMDB-Sent & MNIST \\
 \hline
 Post-hoc accuracy & 0.90.8 & 0.849 &0.958\\
\hline
%  VMI  & 0.222 & 0.287&0.564 \\%[1ex]
% \hline
 Human accuracy& 0.844 & 0.774& NA\\
 \hline  
 \end{tabular} 
 }
 \vspace{-2mm}
 \caption{Post-hoc accuracy and human accuracy of L2X on three models: a word-based CNN model on IMDB, a hierarchical LSTM model on IMDB, and a CNN model on MNIST.}
\label{table:l2x} 
 %\vspace{-4mm}
\end{table} 

%\vspace{-2mm}
\subsection{MNIST} 
%\vspace{-2mm}
%------------------------------------------------------------------------

The MNIST data set contains $28\times 28$ images of handwritten digits~\cite{lecun1998gradient}. We form a subset of the MNIST data set by choosing images of digits $3$ and $8$, with $11,982$ images for training and $1,984$ images for testing.
% The training set is composed of $6,131$ and $5,851$ images labeled with $3$ and $8$ respectively. The test set is composed of $1,010$ and $974$ images labeled with $3$ and $8$ respectively. 
Then we train a simple neural network for binary classification over the subset, which achieves accuracy $99.7\%$ on the test data set. The neural network is composed of two convolutional layers of kernel size $5$ and a dense linear layer at last. The two convolutional layers contains $8$ and $16$ filters respectively, and both are followed by a max pooling layer of pool size $2$. We try to explain each sample image with $k=4$ image patches on the neural network model, where each patch contains $4\times 4$ pixels, obtained by dividing each $28\times 28$ image into $7\times 7$ patches. We use patches instead of raw pixels as features for better 
visualization.   

% We compare L2X with other explainers including Taylor, Saliency, SHAP, and LIME. For Saliency and Taylor, the importance scores of each patch are the $L_1$ norm of the gradient with respect to the $4\times 4$ raw pixels, or its inner product with the corresponding raw pixels respectively. 
% For L2X, we parametrize both the explainer and the variational family with two layers of convolutional networks. 
We parametrize the explainer and the variational family with three-layer and two-layer convolutional networks respectively, with max pooling added after each hidden layer. The $7\times 7$ vector sampled from the explainer is upsampled (with repetition) to size $28\times 28$ and multiplied with the input raw pixels.

We use only the post-hoc accuracy for experiment, with results shown in Table~\ref{table:l2x}. The predictions based on 4 patches selected by L2X out of 49 align with those from original images for $95.8\%$ of data. Randomly selected examples with explanations are shown in Figure~\ref{fig:mnist}. We observe that L2X captures most of the informative patches, in particular those containing patterns that can distinguish 3 and 8. 
% Figure \ref{fig:mnist_results} shows the explanation result for ten randomly selected samples. We can see the distinct topolopy of $3$ and $8$ has been captured by just using ten pixels extracted by our method. 

% , and another common approach which uses soft attention mechanism. For soft attention, we train a hierarchical LSTM with the attention over all sentence representation vectors over the training set labelled by the original model, and pick the sentence with the largest weight as the key sentence for each review. 
% The results of variational-MI are reported in Table \ref{tab:???}, where we use an LSTM followed by a hidden dense layer as our variational family. 

% We sample the most important sentence from the review, and use that for prediction, and get \%80 accuracy on the validation data set. Below is a screenshot of randomly selected ten samples.

% \section{Discussion} 
% regression. 

% We find exactly sampling $k$ distinct features during training may not be necessary, or even incurs the slow convergence of optimization. (mention the work. )

% semi-supervised ... 
%------------------------------------------------------------------------ 
%\vspace{-3mm}
\section{Conclusion} 
%\vspace{-2mm}
%------------------------------------------------------------------------

We have proposed a framework for instancewise feature selection via mutual information, and a method L2X which seeks a variational approximation of the mutual information, and makes use of a
Gumbel-softmax relaxation of discrete subset sampling during training. To our best knowledge, L2X is the first method to realize real-time interpretation of a black-box model. We have shown the efficiency and the capacity of L2X for instancewise feature selection on both synthetic and real data sets.

\newpage
\section*{Acknowledgements}

L.S. was also supported in part by NSF IIS-1218749, NIH BIGDATA 1R01GM108341, NSF CAREER IIS-1350983, NSF IIS-1639792 EAGER, NSF CNS-1704701, ONR N00014-15-1-2340, Intel ISTC, NVIDIA and Amazon AWS. We thank Nilesh Tripuraneni for comments about the Gumbel trick. 
\appendix

\section{Proof of Theorem 1}
\paragraph{Forward direction:} Any explanation is represented as a conditional distribution of the feature subset over the input vector. Given the definition of $S^*$, we have for any $X$, and any explanation $\mathcal E:S|X$,
\begin{align*}
\mathbb E_{S|X}\mathbb E_m [&\log{P_m(Y|X_S)}|X] \leq \\
&\mathbb E_m [\log{P_m(Y|X_{S^*(X)})}|X].
\end{align*}
In the case when $S^*(X)$ is a set instead of a singleton, we identify $S^*(X)$ with any distribution that assigns arbitrary probability to each elements in $S^*(X)$ with zero probability outside $S^*(X)$. With abuse of notation, $S^*$ indicates both the set function that maps every $X$ to a set $S^*(X)$ and any real-valued function that maps $X$ to an element in $S^*(X)$.

Taking expectation over the distribution of $X$, and adding $\mathbb E\log P_m(Y)$ at both sides, we have 
\begin{equation*}
I(X_S;Y)\leq I(X_{S^*};Y)
\end{equation*}
for any explanation $\mathcal E:S|X$. 

\paragraph{Reverse direction:} The reverse direction is proved by contradiction. Assume the optimal explanation $P(S|X)$ is such that there exists a set $M$ of nonzero probability, over which $P(S|X)$ does not degenerates to an element in $S^*(X)$. Concretely, we define $M$ as
\begin{equation*}
M=\{x: P(S \notin S^*(x)|X=x) > 0\}.
\end{equation*} 

For any $x\in M$, we have 
\begin{align}\label{eq:inside_M}
\mathbb E_{S|X}\mathbb E_m [&\log{P_m(Y|X_S)}|X=x]<\nonumber\\
&\mathbb E_m [\log{P_m(Y|X_{S^*(x)})}|X=x],
\end{align}
where $S^*(x)$ is a deterministic function in the set of distributions that assign arbitrary probability to each elements in $S^*(x)$ with zero probability outside $S^*(x)$. Outside $M$, we always have 
\begin{align}\label{eq:outside_M}
\mathbb E_{S|X}\mathbb E_m [&\log{P_m(Y|X_S)}|X=x]\leq\nonumber \\
&\mathbb E_m [\log{P_m(Y|X_{S^*(x)})}|X=x] 
\end{align}
from the definition of $S^*$. As $M$ is of nonzero size over $P(X)$, combining Equation \ref{eq:inside_M} and Equation \ref{eq:outside_M} and taking expectation with respect to $P(X)$, we have 
\begin{equation}
I(X_S;Y)< I(X_{S^*};Y),
\end{equation}
which is a contradiction.

\bibliography{explanation}
\bibliographystyle{icml2018}

%%%%%%%%%%%%%%%%%%%%%%%%%%%%%%%%%%%%%%%%%%%%%%%%%%%%%%%%%%%%%%%%%%%%%%%%%%%%%%%
%%%%%%%%%%%%%%%%%%%%%%%%%%%%%%%%%%%%%%%%%%%%%%%%%%%%%%%%%%%%%%%%%%%%%%%%%%%%%%%
% DELETE THIS PART. DO NOT PLACE CONTENT AFTER THE REFERENCES!
%%%%%%%%%%%%%%%%%%%%%%%%%%%%%%%%%%%%%%%%%%%%%%%%%%%%%%%%%%%%%%%%%%%%%%%%%%%%%%%
%%%%%%%%%%%%%%%%%%%%%%%%%%%%%%%%%%%%%%%%%%%%%%%%%%%%%%%%%%%%%%%%%%%%%%%%%%%%%%% 

% For regression models which only output a number, we assume a Gaussian distribution, or a custom-chosen distribution on the response variable. In addition to the model, we assume access to a data set $\{x_i\}$ where each $x_i$ is independently generated from the distribution $P(X)$ of the input vector $X$. 

\end{document}